\documentclass[twoside,leqno,twocolumn]{article}
\usepackage[colorinlistoftodos]{todonotes}
\usepackage{ltexpprt}
\usepackage{cite}
\usepackage{url}
\usepackage{amsmath,amssymb,amsfonts}
\usepackage{enumitem}
\usepackage[linesnumbered,ruled,vlined,algo2e]{algorithm2e}
\usepackage{graphicx}
\usepackage{textcomp}
\usepackage{xcolor}
\usepackage{tikz}
\usepackage{ifthen}
\usepackage{subcaption}
\usepackage{multirow}
\usepackage{todonotes}
\usepackage[colorlinks=true]{hyperref}
\usepackage{float}

\def\path{\ensuremath{\mathcal{A}}}
\def\pathlength{\ensuremath{P}}

\def\reqsetcost(#1){\ensuremath{\mathcal{C}(#1)}}

\def\downrs(#1){\ensuremath{\mathcal{O}_A(#1)}}
\def\uprs(#1){\ensuremath{\mathcal{O}_B(#1)}}
\def\lowerbound(#1){\ensuremath{\mathrm{LB}(#1)}}
\def\lb(#1,#2){\ensuremath{\textrm{LB}(#1,#2)}}
\def\dtw(#1,#2){\ensuremath{\textrm{DTW}(#1,#2)}}
\def\cdtw(#1,#2){\ensuremath{\textrm{DTW}_W(#1,#2)}}
\def\leftrs(#1){\ensuremath{\mathcal{L}^W_{#1}}}
\def\rightrs(#1){\ensuremath{\mathcal{R}^W_{#1}}}

\def\leftset{\ensuremath{\mathcal{L}}}
\def\rightset{\ensuremath{\mathcal{R}}}
\def\Vset{\ensuremath{\mathcal{V}}}

\newcommand*{\Tpair}{$A$ and $B$}
\newcommand{\lbkim}{\textsc{LB\_Kim}}
\newcommand{\lbyi}{\textsc{LB\_Yi}}
\newcommand{\lbkeogh}{\textsc{LB\_Keogh}}

\newcommand{\lbimproved}{\textsc{LB\_Improved}}
\newcommand{\lbnew}{\textsc{LB\_New}}
\newcommand{\lbtan}{\textsc{LB\_Enhanced}}

\newcommand{\lbtanone}{\textsc{LB\_Enhanced$^1$}}
\newcommand{\lbtantwo}{\textsc{LB\_Enhanced$^2$}}

\newcommand{\lbtanfive}{\textsc{LB\_Enhanced$^5$}}

\DeclareMathOperator{\LBKim}{\textsc{LB\_Kim}}
\DeclareMathOperator{\LBYi}{\textsc{LB\_Yi}}
\DeclareMathOperator{\LBKeogh}{\textsc{LB\_Keogh}}

\DeclareMathOperator{\LBImproved}{\textsc{LB\_Improved}}
\DeclareMathOperator{\LBNew}{\textsc{LB\_New}}
\DeclareMathOperator{\LBTan}{\textsc{LB\_Enhanced}}
\DeclareMathOperator{\DTW}{\textsc{DTW}}
\DeclareMathOperator{\NNDTW}{\textsc{NN-DTW}}

\graphicspath{{images/}}

\def\W{4}

\newcommand{\ourRepo}{\url{http://bit.ly/SDM19}}
\newcommand{\githubRepo}{\url{https://github.com/ChangWeiTan/LbEnhanced}}
\title{Elastic bands across the path: \\ A new framework and method to lower bound DTW}

\author{Chang Wei Tan \\chang.tan@monash.edu \and 
Fran\c{c}ois Petitjean \\francois.petitjean@monash.edu \and  
Geoffrey I.\ Webb \\geoff.webb@monash.edu}
\date{\vspace{-3mm}\small Faculty of IT, Monash University, Melbourne, Australia}

\begin{document}
\maketitle
\begin{abstract}\small\baselineskip=9pt 
The Nearest Neighbour algorithm coupled with the Dynamic Time Warping similarity measure (NN-DTW) is at the core of state-of-the-art classification algorithms including Ensemble of Elastic Distances and Collection of Transformation-Based Ensemble. 
DTW's complexity makes NN-DTW  highly computationally demanding. 
To combat this, lower bounds to DTW are used to minimize the number of times the expensive DTW need be computed during NN-DTW search. 
Effective lower bounds must balance `time to calculate' vs `tightness to DTW.'
On the one hand, the tighter the bound the fewer the calls to the full DTW. 
On the other, calculating tighter bounds usually requires greater computation.
Numerous lower bounds have been proposed. 
Different bounds provide different trade-offs between computational time and tightness. 
In this work, we present a new class of lower bounds that are tighter than the popular Keogh lower bound,
while requiring similar computation time.
Our new lower bounds take advantage of the DTW boundary condition, 
monotonicity and continuity constraints. 
In contrast to most existing bounds, 
they remain relatively tight even for large windows.
A single parameter to these new lower bounds controls the speed-tightness trade-off. 
We demonstrate that these new lower bounds provide an exceptional balance between computation time and tightness for the NN-DTW time series classification task, resulting in greatly improved efficiency for NN-DTW lower bound search.

\end{abstract}

\section{Introduction}
Dynamic Time Warping ($\DTW$) lower bounds play a key role in speeding up many forms of time series
analytics \cite{tan2017indexing,rakthanmanon2012searching,petitjean2014dynamic,keogh2005exact}. 
Several lower bounds have been proposed \cite{shen2018accelerating,keogh2005exact,kim2001index,lemire2009faster,yi1998efficient}.
Each provides a different trade-off between compute time (speed) and tightness. Figure~\ref{fig:motivation} \begin{figure}[t!]
\centering
\includegraphics[width=.95\columnwidth,trim=10pt 0pt 20pt 0pt]{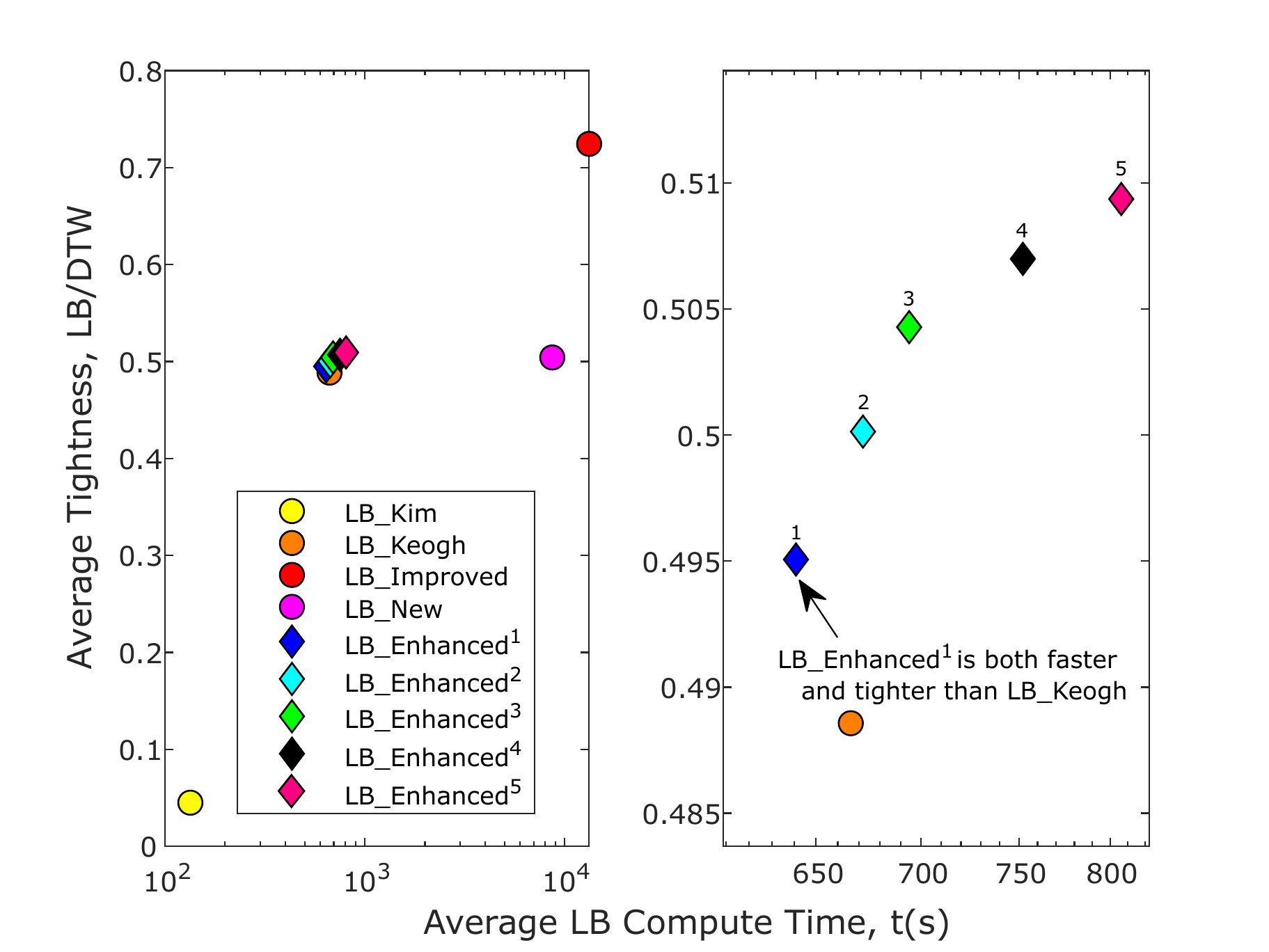}
\vspace*{-6pt}
\caption{Tightness-Compute Time comparison of existing and our lower bounds at $W{=}0.1\cdot L$ over 250,000 time series pairs with $L{=}256$ randomly sampled from the benchmark UCR time series archive \cite{ucrarchive}. Figure on the right shows the zoomed in plot.
Our $\lbtan^1$ is faster and tighter than \lbkeogh{}.
Our other variants are even tighter at a cost of being slightly slower.}
\label{fig:motivation}
\end{figure}%
illustrates this, 
plotting average tightness $(\lb(A,B)/\DTW(A,B))$ against the average time to compute for alternative lower bounds.
As shown in Figure~\ref{fig:tightness},
\begin{figure}[t!]
\centering
\includegraphics[width=0.9\columnwidth]{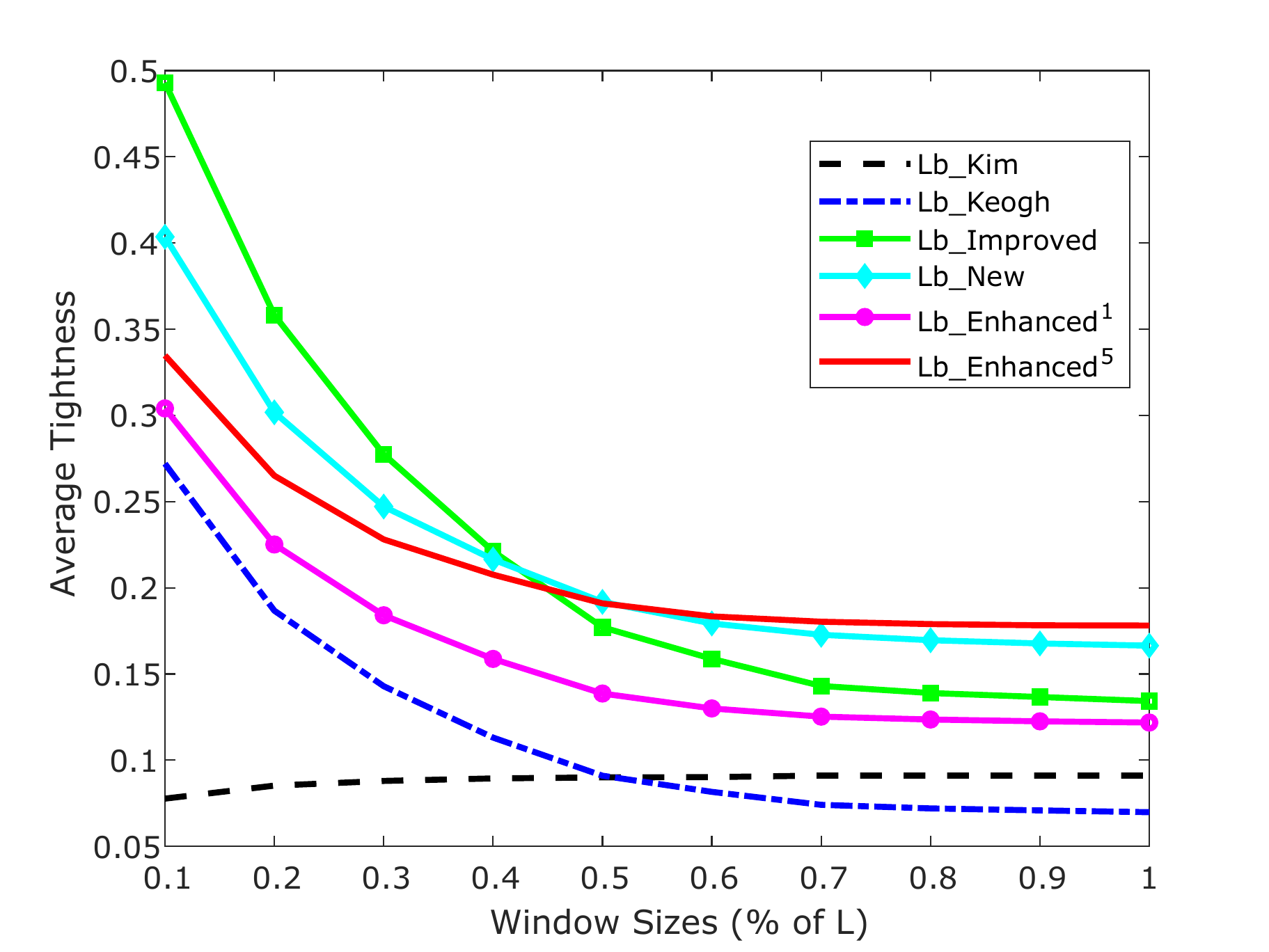}
\vspace*{-6pt}
\caption{Tightness of different lower bounds at differing window sizes averaged across all UCR datasets \cite{ucrarchive}. Our very efficient  $\lbtan^5$ is tighter than any alternative at large window sizes.}\label{fig:tightness}
\end{figure}
different bounds have different relative tightness at different window sizes (the tighter the better).

In this paper, we present a family of lower bounds, 
all of which are of $O(L)$ time complexity and 
are in practice tighter than \lbkeogh{} \cite{keogh2005exact}.  
Our new lower bound is parameterized,
with a tightness parameter, $V$ controlling a useful speed-tightness trade-off.
Two of these, \lbtanone{} and \lbtantwo{}, 
where $V{=}\{1,2\}$ respectively,
have very similar compute time to \lbkeogh{}, 
while providing tighter bounds, 
meaning that their performance should dominate that of \lbkeogh{} on any standard time series analysis task.

We focus on the application of lower bounds to $\DTW$ in Nearest Neighbor ($\NNDTW$) Time Series Classification (TSC). 
$\NNDTW$ is in its own right a useful TSC algorithm and 
is a core component of the most accurate current TSC algorithms,
\textsc{COTE} \cite{bagnall2015time} and 
\textsc{EE} \cite{lines2015time}, 
which are ensembles of TSC algorithms.
$\NNDTW$ is thus at the core of TSC algorithms, 
but is extremely costly to compute \cite{tan2017indexing,tan2018efficient}.
Given a training set with $N$ time series and length $L$, 
a single classification with standard $\NNDTW$ requires $O(N\cdot L^2)$ operations.
Besides, $\NNDTW$ is only competitive when used with a warping window, $W$ learned from the training set \cite{bagnall2017great}. 
Learning the best warping window is very time consuming as it requires the enumeration of 
numerous
windows in the range of 0\% to 100\% of $L$ and is extremely inefficient for large training sets \cite{tan2018efficient}.

There has been much research into speeding up $\NNDTW$, 
tackling either the $N$ part \cite{tan2017indexing} or the 
$L^2$ part of the complexity \cite{keogh2005exact,kim2001index,lemire2009faster,yi1998efficient,xi2006fast,shen2018accelerating}.
A key strategy is to use lower bound search, 
which employs lower bounds on $\DTW$ to efficiently exclude nearest neighbor candidates without having to calculate $\DTW$ \cite{shen2018accelerating,keogh2005exact,kim2001index,lemire2009faster,yi1998efficient}. 
We show that different speed-tightness trade-offs from different lower bounds prove most effective at speeding up $\NNDTW$ for different window sizes. 

Of the existing widely used lower bounds, 
\lbkim{} \cite{kim2001index} is the fastest,
with constant time complexity with respect to window size. 
It is the loosest of the existing standard bounds for very small $W$,
but its relative tightness increases as window size increases.
For small window sizes, 
\lbkeogh{} \cite{keogh2005exact} provides an effective trade off between speed and tightness. 
However, as shown in Figure~\ref{fig:tightness}, it is sometimes even looser than \lbkim{} at large window sizes. 
The more computationally intensive, \lbimproved{} \cite{lemire2009faster} provides a more productive trade-off for many of the larger window sizes.
\lbnew{} \cite{shen2018accelerating} does not provide a winning trade-off for this task at any window size.

The new $\DTW$ lower bound that we propose has the same complexity $O(L)$ as \lbkeogh{}.
At its lowest setting, \lbtanone{}, 
it is uniformly tighter than \lbkeogh{}. 
It replaces two calculations of the distance of a query point (the first and last) to a target \lbkeogh{} envelope with two calculations of distances between a query and a target point. This may or may not be faster, depending whether the query point falls within the envelope, in which case \lbkeogh{} does not perform a distance calculation. However, due to its greater tightness, this variant always supports faster $\NNDTW$.  
At $V{=}\{2,3,4,5\}$, 
our tighter \lbtan{} provide the greatest speed-up out of all standard $\DTW$ lower bounds for $\NNDTW$ over a wide range of window sizes.

Our paper is organised as follows.
In Section~\ref{sec:background},
we review relevant background and related work. 
Then we describe our proposed lower bound in Section~\ref{sec:ourlb}.
Section~\ref{sec:experiments} presents an evaluation of our new lower bound in terms of its utility in TSC with $\NNDTW$. 
Lastly, we conclude our paper in Section~\ref{sec:conclusion}.

\section{Background and Related Work} \label{sec:background}
We let $A=\langle A_1, \ldots, A_L\rangle$ and $B=\langle B_1, \ldots, B_L\rangle$ be a pair of time series \Tpair{} that we want to compare.
Note that, for ease of exposition, 
we assume that the two series are of equal length, 
but the techniques trivially generalize to unequal length series.

\subsection{Dynamic Time Warping} \label{sec:dtw}
The Dynamic Time Warping ($\DTW$) similarity measure was first introduced in \cite{sakoe1972dynamic} for aligning and comparing time series with application to speech recognition.
$\DTW$ finds the global \emph{alignment} of a time series pair, 
\Tpair{}, as illustrated in Figure~\ref{fig:dtw}.
The \emph{warping path} of $A$ and $B$ is a sequence $\path=\langle\path_1,\ldots,\path_\pathlength\rangle$ of \emph{links}. 
Each link is a pair $\path_k=( i,j)$ indicating that $A_i$ is aligned with $B_j$. 
$\path$ must obey the following constraints: 
\begin{itemize}
  \item {\bf Boundary Conditions}: $\path_1=( 1,1)$ and $\path_\pathlength=(L,L)$.
  \item {\bf Continuity} and {\bf Monotonicity}: for any $\path_k=( i,j)$, $1<k\leq\pathlength$, $\path_{k+1}\in\{( i{+}1,j), ( i,j{+}1), ( i{+}1,j{+}1)\}$.
\end{itemize}

\begin{figure}[t]
\centering
\begin{subfigure}[b]{0.49\columnwidth}
\includegraphics[width=\columnwidth]{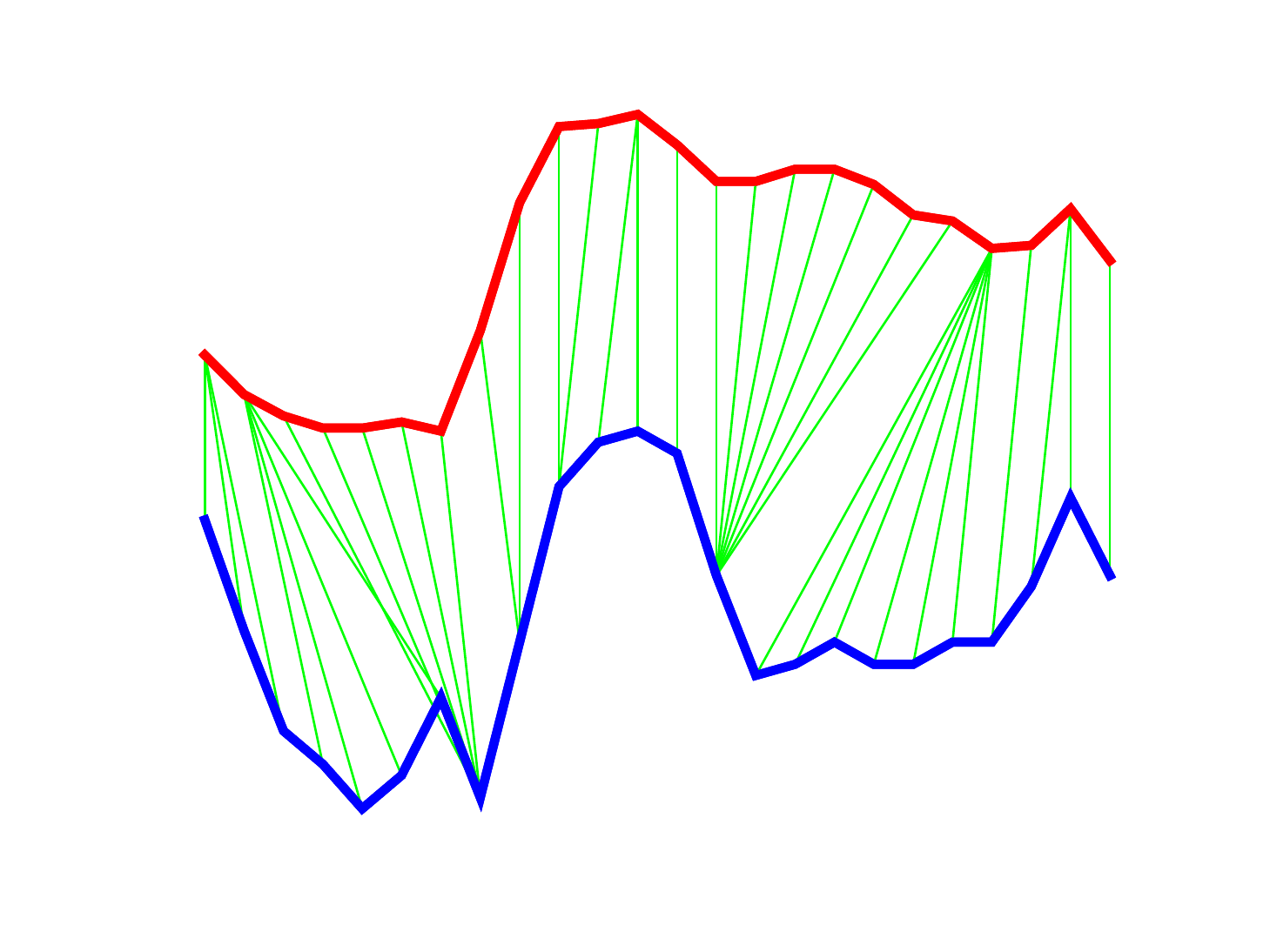}
\caption{}
\label{fig:dtw}
\end{subfigure}
\hfill
\begin{subfigure}[b]{0.49\columnwidth}
\includegraphics[width=\columnwidth]{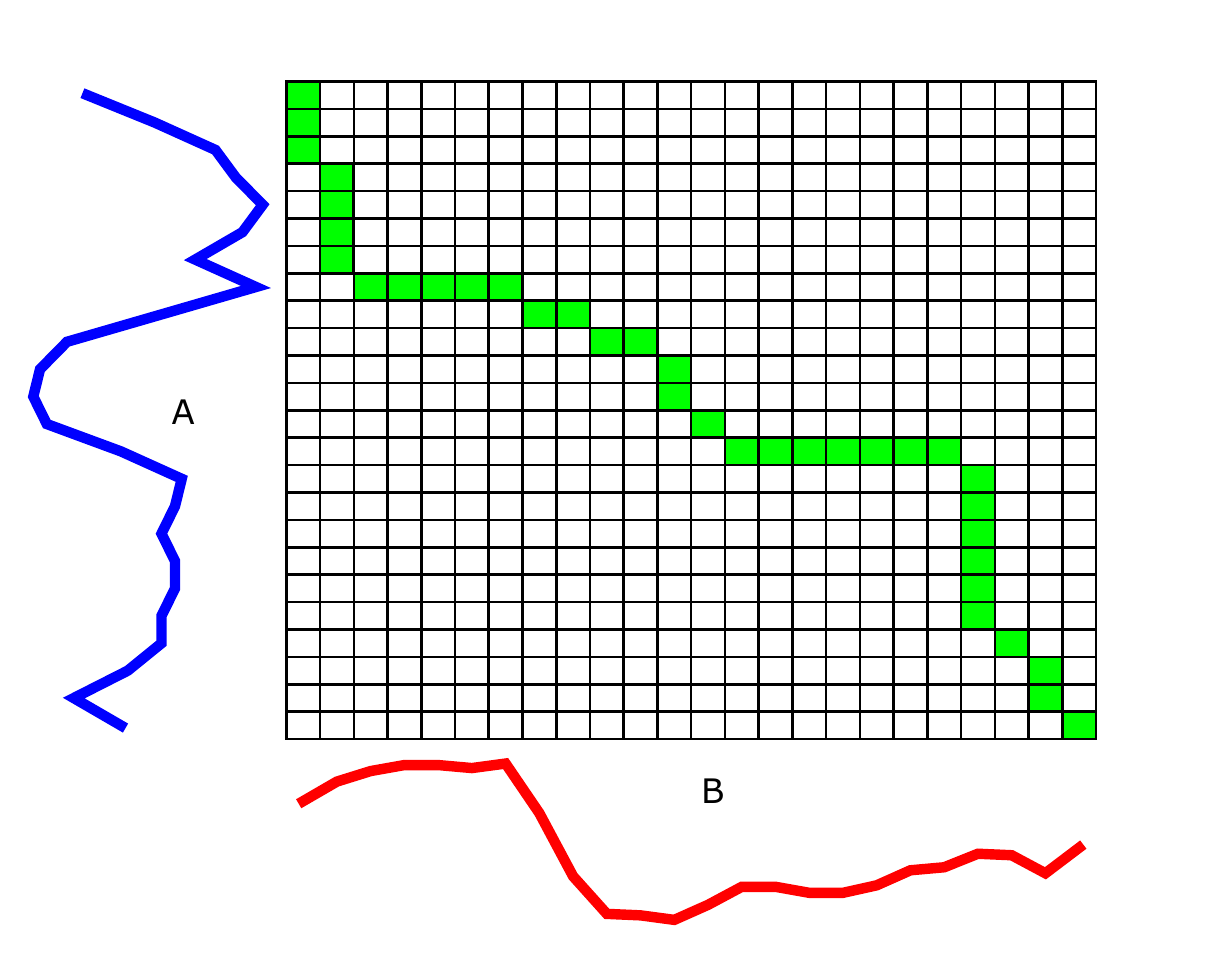}
\caption{}
\label{fig:costmatrix}
\end{subfigure}

\caption{(a) $\DTW$ warping path for time series \Tpair{} (b) Cost matrix $D$ with the warping path $\path$ (green)}
\end{figure}

The cost of a warping path is minimised using dynamic programming by building a cost matrix $D$, as illustrated in Figure~\ref{fig:costmatrix}. 
Each cell $(i,j)$ of the matrix shows the cost of aligning $A_i$ and $B_j$. 
The cost of the warping path for a pair of time series \Tpair{} is computed recursively using Equation~\ref{eqn:cost}, 
where $\delta(A_i,B_j)$ is the $L_2$-norm of $A_i$ and $B_j$. 


\begin{equation} \label{eqn:cost}
    D(i,j) = \delta(A_i, B_j) + \min
    \begin{cases}
        D(i-1,j-1) \\
        D(i,j-1)   \\
        D(i-1,j) 
    \end{cases} 
\end{equation}

\noindent
Then $\dtw(A,B)$ is calculated using Equation~\ref{eqn:dtw}, 
where $\path_i^1$ is the first index in $\path_i$ and $\path_i^2$ the second.

\vspace{-10pt}
\begin{equation} \label{eqn:dtw}
    \DTW(A,B)=\sqrt{D(L,L)}=\sqrt{\sum_{i=1}^L\delta(A_{\path_i^1},B_{\path_i^2})}
\end{equation}

A global constraint on the warping path can be applied to $\DTW$,
such that $A_i$ and $B_j$ can only be aligned if they are within a window range, $W$. 
This limits the distance in the time dimension that can separate $A_i$ from points in $B$ with which it can be aligned \cite{sakoe1972dynamic,keogh2005exact}.  
This constraint is known as the warping window,
$W$ (previously Sakoe-Chiba band) \cite{sakoe1972dynamic}
and we write this as $\DTW_W(A,B)$.
Note that we have $1\leq W\leq L$; 
$\DTW_0$ corresponds to the Euclidean distance; and $\DTW_L$ is equivalent to unconstrained DTW. 
Figure~\ref{fig:ww} shows an example with warping window $W{=}3$,
where the alignment of \Tpair{} is constrained to be inside the gray band. 

\begin{figure}[!t]
\centering
	\includegraphics[width=0.6\columnwidth]{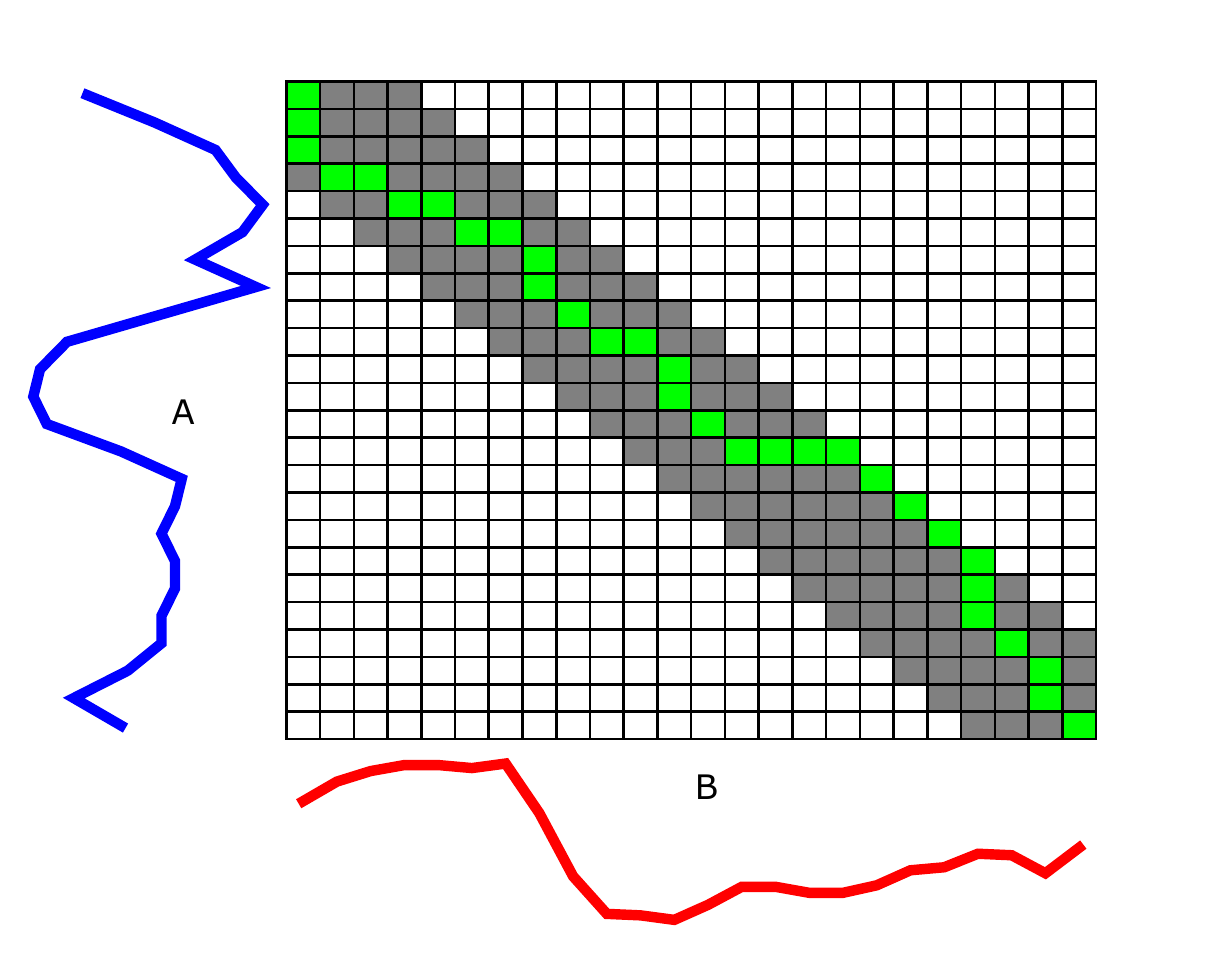}
    \vspace*{-10pt}
    \caption{$\DTW$ with warping window, $W=3$}
    \label{fig:ww}
\end{figure}

Constraining the warping path has two main benefits: 
(1) increasing   $\NNDTW$ accuracy by preventing pathological warping of a pair of time series \Tpair{}.
(2) speeding up $\NNDTW$ by reducing the complexity of $\DTW$ from $O(L^2)$ to $O(W\cdot L)$ \cite{tan2018efficient}.

Alternative constraints  have also been developed,
such as the Itakura Parallelogram \cite{itakura1975minimum}
and the Ratanamahatana-Keogh band \cite{ratanamahatana2004making}. 
In this paper, we focus on the Sakoe-Chiba Band which, 
arguably, is by far the most used constraint in the literature \cite{ratanamahatana2005three,xi2006fast,tan2018efficient}. 

\subsection{Existing DTW Lower Bounds} \label{sec:lb}
$\NNDTW$ with lower bound minimises the number of $\DTW$ computations. 
In this section, we review the existing lower bounds for $\DTW$.
For the rest of the paper, 
we will refer a lower bound as \textsc{LB\_$\langle$name$\rangle$} and
consider $A$ as the query time series that is compared to $B$.

\subsubsection{Kim Lower Bound} \label{sec:lbkim}
(\lbkim) \cite{kim2001index} is the simplest lower bound for $\DTW$ with constant $O(1)$ complexity.
\lbkim{} extracts four features -- distances of the first, last, minimum and maximum points from the time series. 
Then the maximum of all four features is the lower bound for $\DTW$. 

\begin{equation*}
    \LBKim(A,B) = \max
    \begin{cases}
        \delta(A_1,B_1) \\
        \delta(A_L,B_L)   \\
        \delta(A_{\min},B_{\min}) \\ 
        \delta(A_{\max},B_{\max})
    \end{cases} 
    \label{eqn:lbkim}
\end{equation*}

\subsubsection{Yi Lower Bound} \label{sec:lbyi}

(\lbyi) \cite{yi1998efficient} takes advantage that  all the points in $A$ that are larger than $B_{\max}$ or smaller than $B_{\min}$ shown in Equation~\ref{eqn:lbyi},
must at least contribute to the final $\DTW$ distance. 
Thus the sum of their distances to $B_{\max}$ or $B_{\min}$ forms a lower bound for $\DTW$.

\begin{equation*}
    \LBYi(A,B) = \sum_{i=1}^{L}
    \begin{cases}
        \delta(A_i,B_{\max}) & \text{if } A_i > B_{\max} \\
        \delta(A_i,B_{\min}) & \text{if } A_i < B_{\min}  \\
        0 & \text{otherwise}
    \end{cases} 
    \label{eqn:lbyi}
\end{equation*}

\subsubsection{Keogh Lower Bound} \label{sec:lbkeogh}

(\lbkeogh) \cite{keogh2005exact} first creates two new time series, 
upper $\mathbb{U}^B$ (Equation~\ref{eqn:upperenvelope}) and lower $\mathbb{L}^B$ (Equation~\ref{eqn:lowerenvelope}) envelopes. 
These are the upper and lower bounds on $B$ within the window of each point in $A$.
Then the lower bound is the sum of distances to the envelope of points in $A$ that are outside the envelope of $B$.

\begin{gather}
    \mathbb{U}_i^B =\max_{\max(1,i{-}W)\leq j \leq \min(L, i{+}W)}(B_j) \label{eqn:upperenvelope} \\
    \mathbb{L}_i^B =\min_{\max(1,i{-}W)\leq j \leq \min(L, i{+}W)}(B_j) \label{eqn:lowerenvelope}
\end{gather} 
\vspace{10pt}
\begin{equation*}
\LBKeogh_W(A,B) = \sum_{i=1}^{m} 
    \begin{cases}
        \delta(A_i,\mathbb{U}_i^B) & \text{if } A_i > \mathbb{U}_i^B \\
        \delta(A_i,\mathbb{L}_i^B) & \text{if } A_i < \mathbb{L}_i^B \\
        0 & \text{otherwise}
    \end{cases} 
\end{equation*}


\subsubsection{Improved Lower Bound} \label{sec:lbimproved} 

(\lbimproved) \cite{lemire2009faster} first computes \lbkeogh{} and finds $A'$, the projection of $A$ on to the envelope of $B$ using Equation~\ref{eqn:projection}, 
where $\mathbb{U}_i^B$ and $\mathbb{L}_i^B$ are the envelope of $B$. 
Then it builds the envelope for $A'$ and computes \lbkeogh$(B, A')$.
Finally, \lbimproved{} is the sum of the two \lbkeogh{}s. 

\vspace{-5pt}
\begin{equation}
    A'_i = 
    \begin{cases}
        \mathbb{U}^B_i & \text{if } A_i > \mathbb{U}^B_i \\
        \mathbb{L}^B_i & \text{if } A_i < \mathbb{L}^B_i \\
        A_i & \text{otherwise}
    \end{cases} \label{eqn:projection}
    \end{equation}
\begin{equation*} 
\begin{split}
        \LBImproved_W(A,B) = \LBKeogh_W(A,B) +\\
         \LBKeogh_W(B,A') 
    \end{split}
\end{equation*}

\lbimproved{} is tighter than \lbkeogh{}, 
but has higher computation overheads, 
requiring multiple passes over the series.
However, to be effective, 
it is usually used with an early abandon process,
whereby the bound determined in the first pass is considered, 
and if it is sufficient to abandon the search, 
the expensive second pass is not performed. 

\subsubsection{New Lower Bound} \label{sec:lbnew} 

(\lbnew{}) \cite{shen2018accelerating} takes advantage of the boundary and continuity conditions for a $\DTW$ warping path to create a tighter lower bound than \lbkeogh. 
The boundary condition requires that every warping path contains $(A_1,B_1)$ and $(A_L,B_L)$.
The continuity condition ensures that every $A_i$ is paired with at least one $B_j$ in $\mathcal{B}_i$,
where $j\in\{\max(1,i-W)\ldots \min(L,i+W)\}$.

\vspace{-10pt}
\begin{align*}
    \LBNew_W(A,B) = \,& \delta(A_1,B_1) + \delta(A_L,B_L) + \\ 
    & \sum_{i=2}^{L-1}\min_{b \in \mathcal{B}_i} \delta(A_i, b) \nonumber \label{eqn:lbnew}
\end{align*}

\subsubsection{Cascading Lower Bounds}
Instead of using standalone lower bounds,
multiple lower bounds with increasing complexity can be cascaded,
to form an overall tighter lower bound \cite{rakthanmanon2012searching}. 
This greatly increases the pruning power and reduces the overall classification time.
The \textsc{UCR Suite} \cite{rakthanmanon2012searching} cascades \lbkim{}, \lbkeogh{}$(A,B)$ and \lbkeogh{}$(B,A)$ to achieve a high speed up in time series search.

\section{Proposed DTW Lower Bound} \label{sec:ourlb}
Our proposed lower bounds are based on the observation that the warping paths are very constrained at the start and the end of the series. 
Specifically, the boundary constraints require that the first link, $\path_1$, is $(1,1)$, the top left of a cost matrix.
The continuity and monotonicity constraints ensure that  $\path_2\in\{(1,2),(2,1),(2,2)\}$.   
If we continue this sequence of sets we get the \emph{left bands},
\begin{align}
\leftset^W_i{=}&\{(\max(1,i{-}W),i),(\max(1,i{-}W)+1,i),\ldots,(i,i),\nonumber\\
&\hspace*{5pt}(i,i-1),\ldots,(i,\max(1,i{-}W)\} \nonumber.
\end{align}
These are the alternating bands through the cost matrix shown in Figure \ref{fig:leftmatrix}. 
We can use these bands to define a lower bound on $\DTW$ as explained in Theorem~\ref{thm:left}.

\begin{theorem}\label{thm:left}
$\sum_{i=1}^L \min_{(j,k)\in \leftset^W_i}\delta(A_j,B_k)$ is a lower bound on $\DTW_W(A,B)$.
\end{theorem}

\begin{proof}
The continuity constraint requires that for all $1\leq i\leq L$, any warping path $\path$ must include $(i,p)$ and $(q,i)$, for some $i{-}W\leq p\leq i{+}W$ and $i{-}W\leq q\leq i{+}W$.  Either the indexes for both $A$ and $B$ reach $i$ in the same pair and $p=q=i$, or one of the indexes must reach $i$ before the other, and $p<q$ or $p>q$. If $p=q=i$, $(i,i)\in\path$.  If $p<q$, $\path$ must contain one of $(i,\max(1,i{-}W)), \ldots (i,i{-}1)$. If $p>q$, $\path$ must contain one of $(\max(1,i{-}W), i), \ldots (i{-}1, i)$. Thus, $\path$ must contain (at least) one of $\leftset^W_i$. It follows that

\vspace{-10pt}
\begin{align*}
    \DTW_W(A,B)&=\sum_{(j,k)\in\path}\delta(A_{j},B_{k})\\
    &=\sum_{i=1}^L\sum_{(j,k)\in \leftset^W_i\cap\path}\delta(A_j,B_k)\\
    &\geq\sum_{i=1}^L \min_{(j,k)\in \leftset^W_i}\delta(A_j,B_k).
\end{align*}
$\square$
\end{proof}

Figure~\ref{fig:leftmatrix} illustrates this lower bound in terms of the cost matrix with $W{=}\W$.
The columns are elements of $A$ and rows the elements of $B$.
The elements in the matrix show the pairwise distances of each point in the time series pair \Tpair{}. 
Successive $\leftset^W_i$ are depicted in alternating colors. 
The minimum distance in each $\leftset^W_i$ is set in bold type. 
The sum of these minimums provides a lower bound on the $\DTW$ distance.

\begin{figure}[t]
    \centering
    \includegraphics[width=0.62\columnwidth]{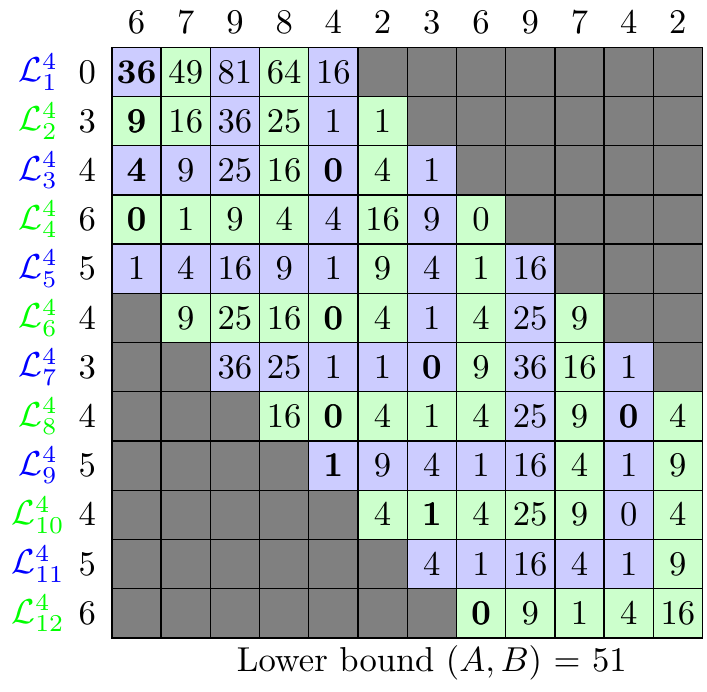}
\vspace*{-8pt}
    \caption{The cost matrix for calculating a lower bound using \emph{left} bands with $W{=}\W$.  
    }.
    \label{fig:leftmatrix}
\end{figure}

Working from the other end, the boundary constraints require that $\path_\pathlength{=}(L,L)$, the bottom right of a cost matrix.
Continuity and monotonicity constraints ensure that at least one of the \emph{right} band 
\begin{align}
\rightset^W_i=&\{(\min(L,i{+}W),i),(\min(L,i{+}W)+1,i),\ldots,(i,i),\nonumber\\
&\hspace*{5pt}(i,i-1),\ldots,(i,\min(L,i{+}W)\} \nonumber
\end{align} is in every warping path. 
Thus, 
\vspace{-5pt}
\begin{equation}
    \DTW_W(A,B)\geq \sum_{i=1}^L \min_{(j,k)\in \rightset^W_{i}}\delta(A_j,B_k).
\end{equation}

Figure~\ref{fig:rightmatrix} illustrates this lower bound in terms of the cost matrix.  
The proof of correctness of this bound is a trivial variant of the proof for Theorem \ref{thm:left}.

\begin{figure}[t]
    \centering
    \includegraphics[width=0.64\columnwidth]{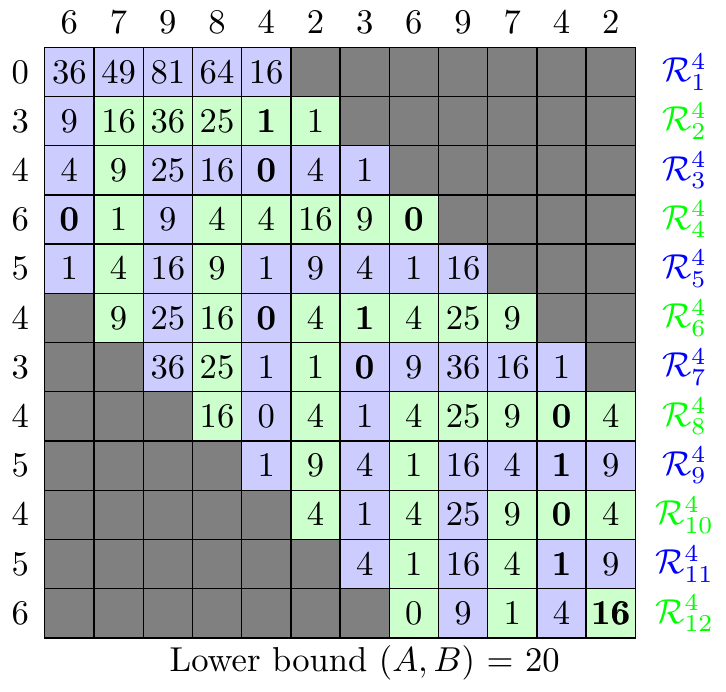}
\vspace*{-8pt}
    \caption{The cost matrix for calculating a lower bound using \emph{right} bands with $W{=}\W$. 
    }
    \label{fig:rightmatrix}
\end{figure}

On the other hand,
\lbkeogh{} uses the minimum value from each band in Figure \ref{fig:keogh-matrix} so long as $A_i>\mathbb{U}_i$ or $A_i<\mathbb{L}_i$.
When $\mathbb{L}_i^B\leq A_i\leq \mathbb{U}_i^B$, 
the band is set in gray. 
For other bands, the minimum is set in bold.
Then the sum over all minimum distances in non-gray bands gives \lbkeogh{}.
It is notable that the leftmost of the left bands and the rightmost of the right bands contain fewer distances than any of the \lbkeogh{} bands.  
All things being equal, 
on average the minimum of a smaller set of distances should be greater than the minimum of a larger set.
Further, because there are fewer distance computations in these few bands, 
it is feasible to take the true minimum of the band,
rather than taking an efficiently computed lower bound on the minimum, as does \lbkeogh{}.

\begin{figure}[t]
    \centering
    \includegraphics[width=.56\columnwidth]{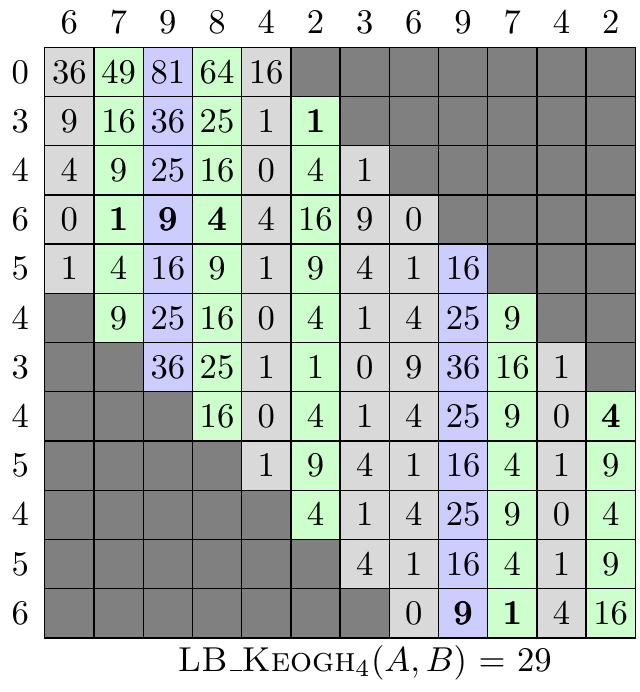}
\vspace*{-8pt}
\caption{The cost matrix for calculating $\LBKeogh_{\W}(A,B)$ with $W{=}\W$.
    } 
    \label{fig:keogh-matrix}
\end{figure}

\subsection{Enhanced Lower Bound} \label{sec:lbtan}
Based on these observations, 
our proposed lower bound exploits the tight leftmost and rightmost bands, 
but uses the \lbkeogh{} bands in the centre section where the left and right bands are larger,
and hence less tight and more expensive to compute.  
This is illustrated in Figure \ref{fig:tan-matrix}. 

\begin{figure}[t]
    \centering
    \includegraphics[width=.58\columnwidth]{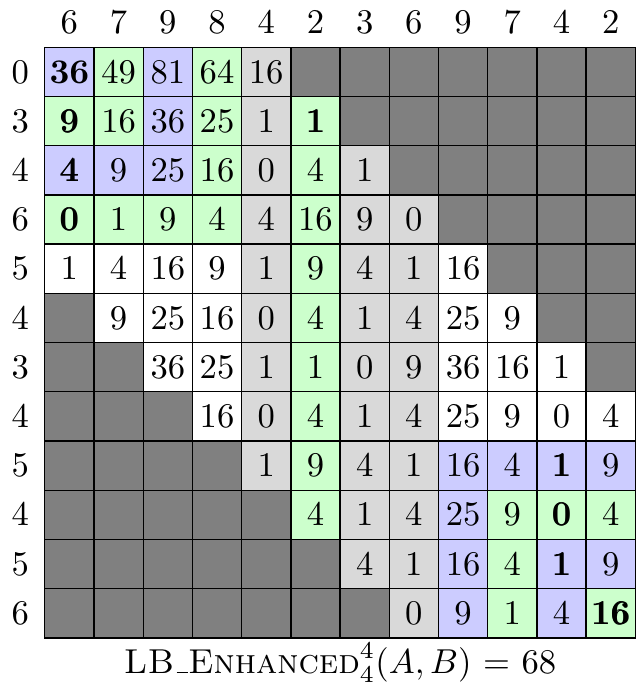}
\vspace*{-8pt}
    \caption{Cost matrix for calculating $\LBTan_4^4(A,B)$ with $V{=}4$ and $W{=}4$. 
    }
    \label{fig:tan-matrix}
\end{figure}
\lbtan{} is parametrized by a tightness parameter $V$, 
$1\leq V\leq L/2$, that specifies how many \emph{left} and \emph{right} bands are utilized.
This controls the speed-tightness trade-off. 
Smaller $V$s require less computation, 
but usually result in looser bounds, 
while higher values require more computation, but usually provide tighter bounds,
as illustrated in Figure \ref{fig:motivation}.
\lbtan{} is defined as follows 
\vspace{-8pt}
\begin{multline}
\LBTan^V_W(A,B)=\\
\sum_{i=1}^L
\begin{cases}
\min_{(j,k)\in\leftrs(i)}\delta(A_j,B_k) & \text{if } i\leq V \\
\min_{(j,k)\in\rightrs(i)}\delta(A_j,B_k) & \text{if } i> L-V \\
\delta(A_i,\mathbb{U}_i^B) & \text{if } A_i>\mathbb{U}_i^B\\
\delta(A_i,\mathbb{L}_i^B) & \text{if } A_i<\mathbb{L}_i^B\\
0 & \text{otherwise}
\end{cases} 
\label{eqn:lbtan}
\end{multline}

\noindent
where $\mathbb{U}_i^B$ and $\mathbb{L}_i^B$ are defined in Equations~\ref{eqn:upperenvelope} and \ref{eqn:lowerenvelope} respectively.

\begin{theorem}
For any two time series $A$  and $B$ of  length $L$, 
for any warping window, $W\leq L$,
and for any integer value $V\leq L/2$,
the following inequality holds: $\LBTan_W^V(A,B)\leq \DTW_W(A,B)$
\end{theorem}

\vspace{-10pt}
\begin{proof}
From the proof for Theorem \ref{thm:left}, for every $1\leq i \leq L$, $\path$ must contain (at least) one of $\leftset^W_i$ and with trivial recasting this also establishes that for every $1\leq i \leq L$, $\path$ must contain (at least) one of $\rightset^W_i$. 

To address the contribution of the $\lbkeogh$ inspired bridge between the $\leftset$s and $\rightset$s,
we introduce the notion of a vertical band $\Vset_i$ for element $A_i$.  
$\Vset_i=\{(i,j):\max(1,i-W)\leq j \leq \min(L, i+W)\}$ is the set of pairs containing $A_i$ that may appear in a warping path.
Note that $\leftset^W_1,\ldots,\leftset^W_V$, $\Vset_{V+1},\ldots,\Vset_{L-V}$ and $\rightset^W_{L-V+1},\ldots,\rightset^W_L$ are all mutually exclusive. None of these sets intersects any of the others. It follows 

\vspace{-10pt}
\begin{align*}
    \DTW_W&(A,B)\\
    &=\sum_{(j,k)\in\path}\delta(A_{j},B_{k})\\
    &\geq\sum_{i=1}^V\sum_{(j,k)\in \leftset^W_i\cap\path}\delta(A_j,B_k)\\
    &\hspace*{10pt}+ \sum_{i=V+1}^{L-V}\sum_{(j,k)\in \Vset_i\cap\path}\delta(A_j,B_k)\\
    &\hspace*{10pt}+ \sum_{i=L-V+1}^{L}\sum_{(j,k)\in \rightset^W_i\cap\path}\delta(A_j,B_k)\\
    &\geq\sum_{i=1}^V \min_{(j,k)\in \leftset^W_i}\delta(A_j,B_k)\\
    &\hspace*{10pt}+ \sum_{i=V+1}^{L-V} \min_{(j,k)\in \Vset_i}\delta(A_j,B_k)\\
    &\hspace*{10pt}+ \sum_{i=L-V+1}^L \min_{(j,k)\in \rightset^W_i}\delta(A_j,B_k)\\
     &\geq\sum_{i=1}^L 
    \begin{cases}
        \min_{(j,k)\in\leftrs(i)}\delta(A_j,B_k) & \text{if } i\leq V \\
        \min_{(j,k)\in\rightrs(i)}\delta(A_j,B_k) & \text{if } i> L-V \\
        \delta(A_i,\mathbb{U}_i^B) & \text{if } A_i>\mathbb{U}_i^B\\
        \delta(A_i,\mathbb{L}_i^B) & \text{if } A_i<\mathbb{L}_i^B\\
        0 & \text{otherwise}
    \end{cases} 
\end{align*}
$\square$
\end{proof}

To illustrate our approach, 
we present the differences between \lbkeogh{} and $\lbtan^4_{\W}$ with respect to $A$ in Figures \ref{fig:keogh-matrix} and \ref{fig:tan-matrix} respectively.
 
In Figure \ref{fig:keogh-matrix}, 
the $i^{th}$ column represents $\Vset_i$, the possible pairs for $A_i$.
The columns are greyed out if $\mathbb{L}_i^B\leq A_i\leq \mathbb{U}_i^B$, showing that they do not contribute to \lbkeogh{}.
For the remaining columns, the numbers in bold are the minimum distance of $A_i$ to a $B_j$ within $A_i$'s window, either $\mathbb{L}_i^B-A_i$ or $A_i-\mathbb{U}_i^B$. The lower bound is the sum of these values.

In Figure \ref{fig:tan-matrix},
alternate bands are set in alternating colors. 
The columns are greyed out if $V<i\leq L-V$ and $\mathbb{L}_i^B\leq A_i\leq \mathbb{U}_i^B$, 
showing that they do not contribute to $\lbtan^4_{\W}$. 
For the remaining columns, 
the numbers in bold are the minimum distance of $A_i$ to $B_j$ within the band.
The lower bound is the sum of these values.
These figures clearly show the differences of \lbkeogh{} and $\lbtan^4_{\W}$, where we take advantage of the tighter \emph{left} and \emph{right} bands. 

We apply a simple technique to make \lbtan{} more efficient and faster. 
In the na\"ive version, \lbtan{} has to compute the minimum distances of \leftset{}s and \rightset{}s.
Usually, these computations are very fast as \leftset{}s and \rightset{}s are much smaller compared to \Vset{}, especially when $L$ is long.
To optimise \lbtan{}, we can first  sum  the minimum distances for the \leftset{}s and \rightset{}s.
Then, if this sum is larger than the current distance to the nearest neighbor, $\mathcal{D}$, we can abort the computation for \Vset{}. 

Algorithm~\ref{alg:lbtan} describes our proposed lower bound. 
\begin{algorithm2e}[t]
\caption{\lbtan{}$(A,B,\mathbb{U}_i^B,\mathbb{L}_i^B,W,V,\mathcal{D})$}
\label{alg:lbtan}
\DontPrintSemicolon
\SetInd{0em}{0.75em}
\KwIn{$A$: Time series A}
\KwIn{$B$: Time series B}
\KwIn{$\mathbb{U}_i^B$: Upper Envelope for B}
\KwIn{$\mathbb{L}_i^B$: Lower Envelope for B}
\KwIn{$W$: Warping window}
\KwIn{$V$: Speed-Tightness parameter}
\KwIn{$\mathcal{D}$: Current distance to NN}
$res\gets \delta(A_1,B_1) + \delta(A_L,B_L)$ \\
$n\_bands\gets \min(L/2, V)$\\
\tcp{Do \leftset{}, \rightset{} bands} 
\For{$i \gets 2$ \KwTo $n\_bands$}{
	$min_{\leftset}\gets \delta(A_i,B_i)$ \\
    $min_{\rightset}\gets \delta(A_{L-i+1},B_{L-i+1})$ \\
    \For{$j\gets \max(1, i-W)$ \KwTo $i-1$}{
    	$min_{\leftset}\gets \min(min_{\leftset},\delta(A_i,B_j))$ \\
    	$min_{\leftset}\gets \min(min_{\leftset},\delta(A_j,B_i))$ \\
        $min_{\rightset}\gets \min(min_{\rightset},\delta(A_{L{-}i{+}1},B_{L{-}j{+}1}))$ \\
        $min_{\rightset}\gets \min(min_{\rightset},\delta(A_{L{-}j{+}1},B_{L{-}i{+}1}))$ \\
    }
    $res\gets res+min_{\leftset}+min_{\rightset}$ \\
}

\lIf{$res \geq \mathcal{D}$}{\Return $\infty$}
\tcp{Do \lbkeogh{}} 
\For{$i \gets n\_bands+1$ \KwTo $L-n\_bands$}{
	\lIf{$A_i> \mathbb{U}_i^B$}{$res\gets res + \delta(A_i,\mathbb{U}_i^B)$}
    \lElseIf{$A_i< \mathbb{L}_i^B$}{$res\gets res + \delta(A_i,\mathbb{L}_i^B)$}
}
\Return $res$
\end{algorithm2e}%
First, we compute the distance of the first and last points as set by the boundary condition. 
In line 2, we define the number of \leftset{} and \rightset{} bands to utilise. 
This number depends on the warping window, $W$, as we can only consider the points within $W$ no matter how big $V$ is.
Line 3 to 11 computes the sum of the minimum distances for \leftset{} and \rightset{}.
If the sum is larger than the current distance to the nearest neighbor, we skip the computation in line 12.
Otherwise, we do standard \lbkeogh{} for the remaining columns in lines 13 to 15.







\section{Empirical Evaluation} \label{sec:experiments}
Our experiments are divided into two parts.
We first study the effect of the tightness parameter $V$ in \lbtan{}. 
Then we show how well \lbtan{} can speed up $\NNDTW$ compared to the other lower bounds.
We used all the 85 UCR benchmark datasets \cite{ucrarchive} and the given train/test splits.
The relative performance of different lower bounds varies greatly with differing window sizes. 
In consequence we conduct experiments across a variety of different window sizes, drawn from two sets of values.
The set $W{=}\{1,\ldots,10\}$, 
spans the best warping windows for most of the UCR benchmark datasets. 
The set, $W{=}\{0.1\cdot L,0.2\cdot L,\ldots,L\}$ shows that using $\NNDTW$ with \lbtan{} is always faster across the broad spectrum of all possible windows.

A $\NNDTW$ with lower bound search can be further sped up by ordering the candidates in the training set based on a proxy for their relative distances to the query, such as a lower bound on that distance \cite{tan2017indexing}. 
However, using a lower bound to order the candidates would unfairly advantage whichever bound was selected.
Hence we order the training set by their Euclidean distance to the query time series (an upper bound on their true distance) and start with the candidate that gives the smallest Euclidean distance. 
Note that our \lbtan{} has even greater advantage if random order is employed.

All experiments were optimised and implemented in Java 8 and conducted on a 64-bit Linux AMD Opteron 62xx Class CPU @2.4GHz machine with 32GB RAM and 4 CPUs.
Our source code are open-source at \githubRepo{} and
the full results at \ourRepo{}.


\subsection{How to choose the right tightness parameter for \lbtan{}?}
Recall that our \lbtan{} is parameterized by a tightness parameter $V$ that specifies the number of bands to be used. 
This parameter controls the speed tightness trade-off.
Higher values require more computations but usually gives tighter bounds. 
We conducted a simple experiment by recording the classification time of $\NNDTW$ with \lbkeogh{} as the baseline and 
\lbtan{} with different tightness parameter in the range of $V{=}\{1,\ldots,20\}$. 
Note that all the required envelopes have been pre-computed at training time and the time is not included in the classification time. 
Then the classification time of $\NNDTW$ with \lbtan{} is normalised by \lbkeogh{}.
Finally the geometric mean is computed over all 85 datasets. 

The results are presented in Figure~\ref{fig:geomean normalised by keogh} where we show the performance for a subset of windows.
The x axis shows the different $V$s, 
the y axis shows the geometric mean of the normalised time.
Ratios below 1 (under the red line) means that \lbtan{} is faster than \lbkeogh{} and 
smaller ratio means faster \lbtan{}. 

\begin{figure}[t]
    \centering
    \includegraphics[width=0.8\columnwidth,trim=10pt 0pt 10pt 10pt]{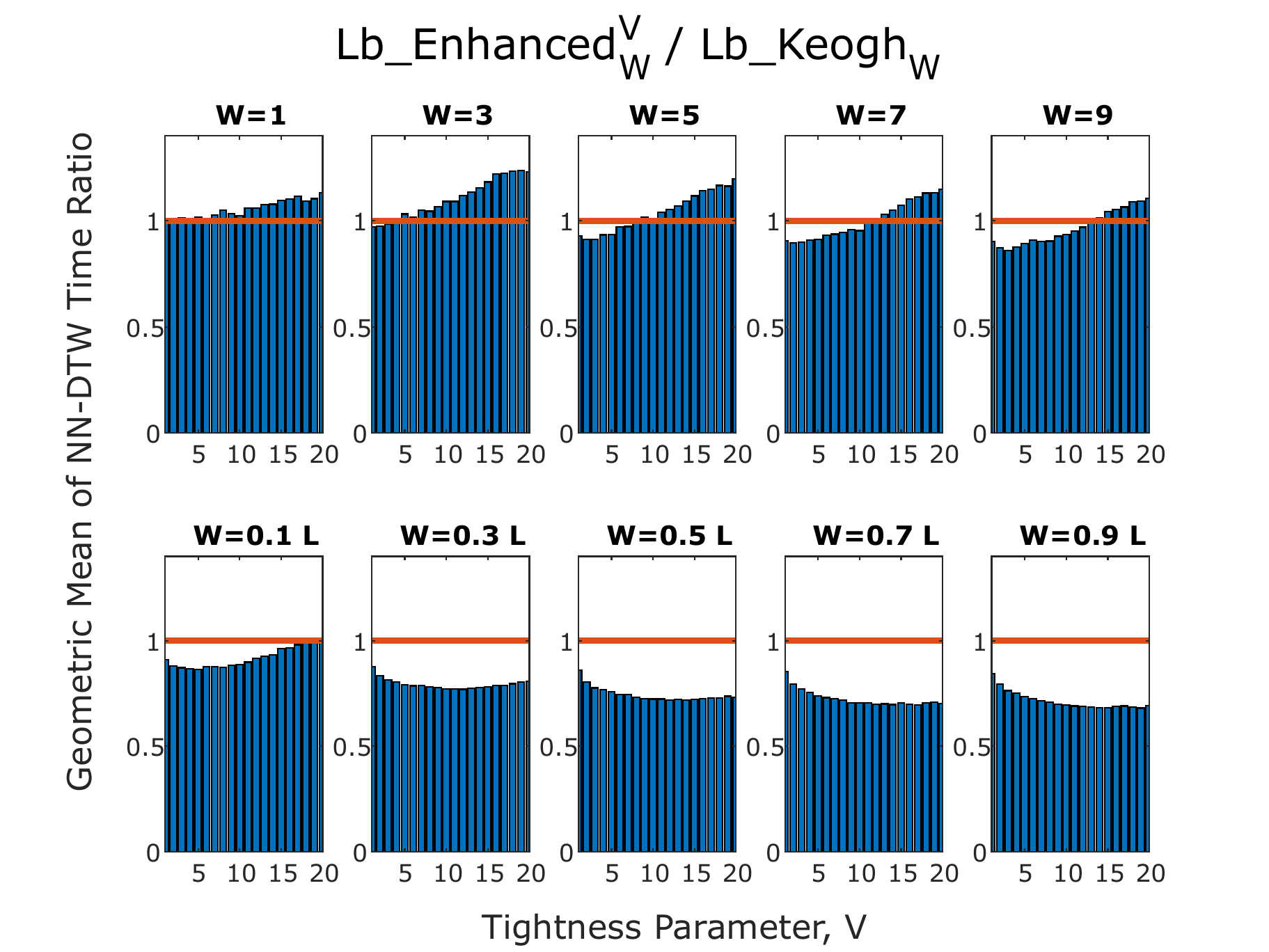}
    \caption{Geometric mean of $\NNDTW$ classification time with \lbtan{} of different tightness parameter $V$ normalised by \lbkeogh{}.
    Smaller ratio means faster, and values below 1 (under the red line) indicate \lbtan{} is faster than \lbkeogh{}.}
    \label{fig:geomean normalised by keogh}
\end{figure}

These plots show that the optimal value of $V$ increases with $W$.
At $W{=}1$, only $V{=}1$ and $V{=}2$ outperform \lbkeogh.
At $W{=}3$, all $V<5$ prove to speed up $\NNDTW$ more than \lbkeogh{}, and subsequently $V<10$ for $W{=}5$ and $V<15$ for $W{=}7$ and $W{=}9$. 
At larger window sizes, 
all $V<20$ are more effective at speeding up $\NNDTW$ than \lbkeogh{}. 
For $W{=}0.1\times L$, $V=5$ proves to be most effective. 
For $W\geq 0.3\times L$,
there are a wide range of values of $V$ with very similar performance.
One reason for this is that window size is not the only factor that affects the optimal value of $V$. Series length and the amount of variance in the sequence prefixes and suffixes are further relevant factors. 
$V{=}5$ provides strong performance across a wide range of window sizes.
In consequence, in the next section we choose $V{=}5$ and compare the performance of \lbtanfive{} to other lower bounds.

\subsection{Speeding up NN-DTW with \lbtan{}}
We compare our proposed lower bounds against four key existing alternatives, 
in total 5 lower bounds: 
\begin{itemize}

\item \textbf{\lbkim{}: }The original \lbkim{} proposed in \cite{kim2001index} is very loose and incomparable to the other lower bounds. 
To make it tighter and comparable, instead of the maximum, 
we take the sum of all the four features without repetitions.
\item \textbf{\lbkeogh{}: }We use the original implementation of \lbkeogh{} proposed in \cite{keogh2005exact}.
\item \textbf{\lbimproved{}: }We use the original implementation of \lbimproved{} and the optimised algorithm to compute the projection envelopes for \lbimproved{} proposed in \cite{lemire2009faster}.
\item \textbf{\lbnew{}: }We use the original implementation of \lbnew{} proposed in \cite{shen2018accelerating}. 
\item \textbf{\lbtan{}: }We use \lbtanfive{}, selecting $V{=}5$ as it provides reasonable speed up across a wide range of window sizes.
\end{itemize}

Note that \lbyi{} was omitted because it is similar to \lbkeogh{} when $W{=}L$.
Similar to before, we record the classification time of $\NNDTW$ with the various lower bounds.
For each dataset we determine the rank of each bound, the fastest receiving rank 1 and the slowest rank 5. 
Figure~\ref{fig:cd time} shows the critical difference (CD) diagram comparing the classification time ranks of each lower bound for $W{=}\{3,6,10,0.1\cdot L,0.5\cdot L\}$.
The results for the remaining of the windows can be found in our supplementary paper and \ourRepo{}.
Each plot shows the average rank of each lower bound (the number next to the name). 
Where the ranks are not significantly different (difference less than CD), their corresponding lines are connected by a solid black line \cite{demvsar2006statistical}. 
Thus, for $W{=}10$, \lbtanfive{} has significantly lower average rank than any other bound and the average ranks of \lbkeogh{} and \lbimproved{} do not differ significantly but are significantly lower than those of \lbnew{} and \lbkim{}; which in turn do not differ significantly from one another.

\begin{figure}[t]
    \centering
    \begin{subfigure}[b]{0.9\columnwidth}
        \centering
        \includegraphics[width=\columnwidth, trim=70pt 165pt 20pt 20pt]{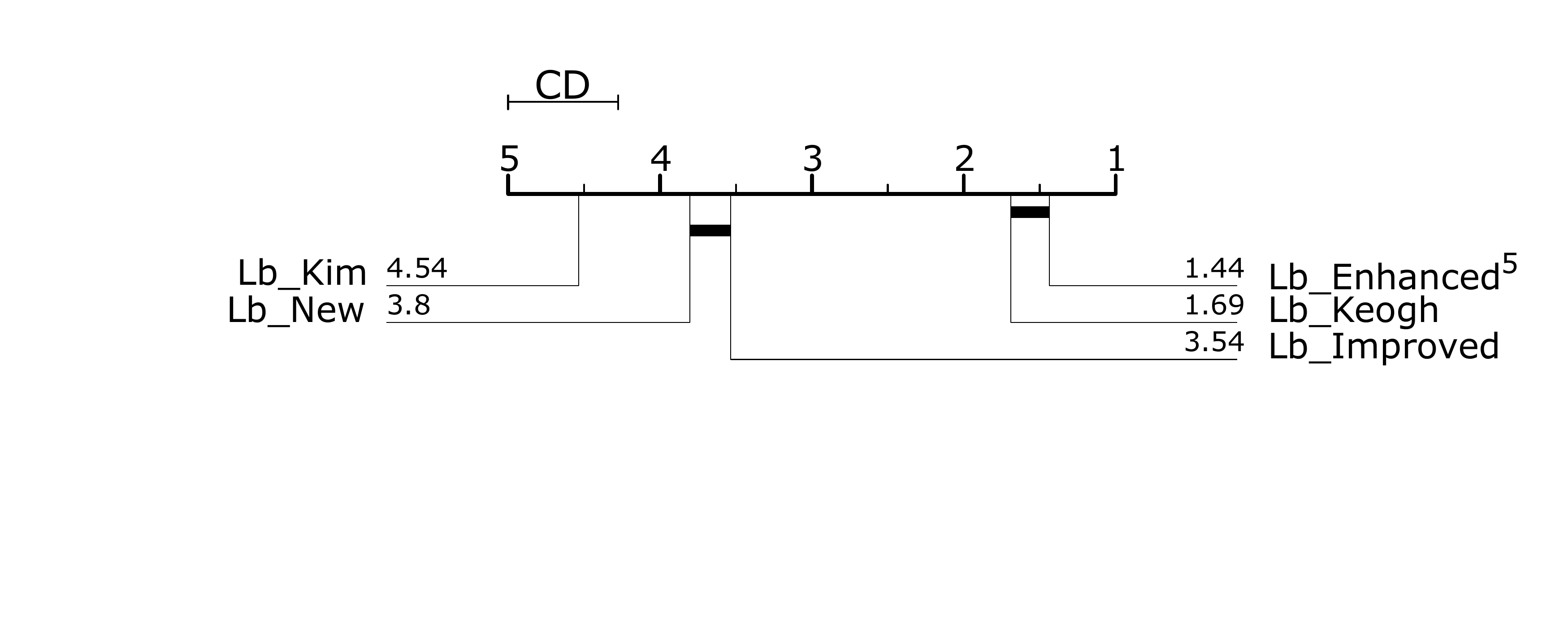}
        \caption{$W{=}3$}
        \label{fig:time w3}
    \end{subfigure}
    \begin{subfigure}[b]{0.9\columnwidth}
        \includegraphics[width=\columnwidth, trim=70pt 175pt 20pt 0pt]{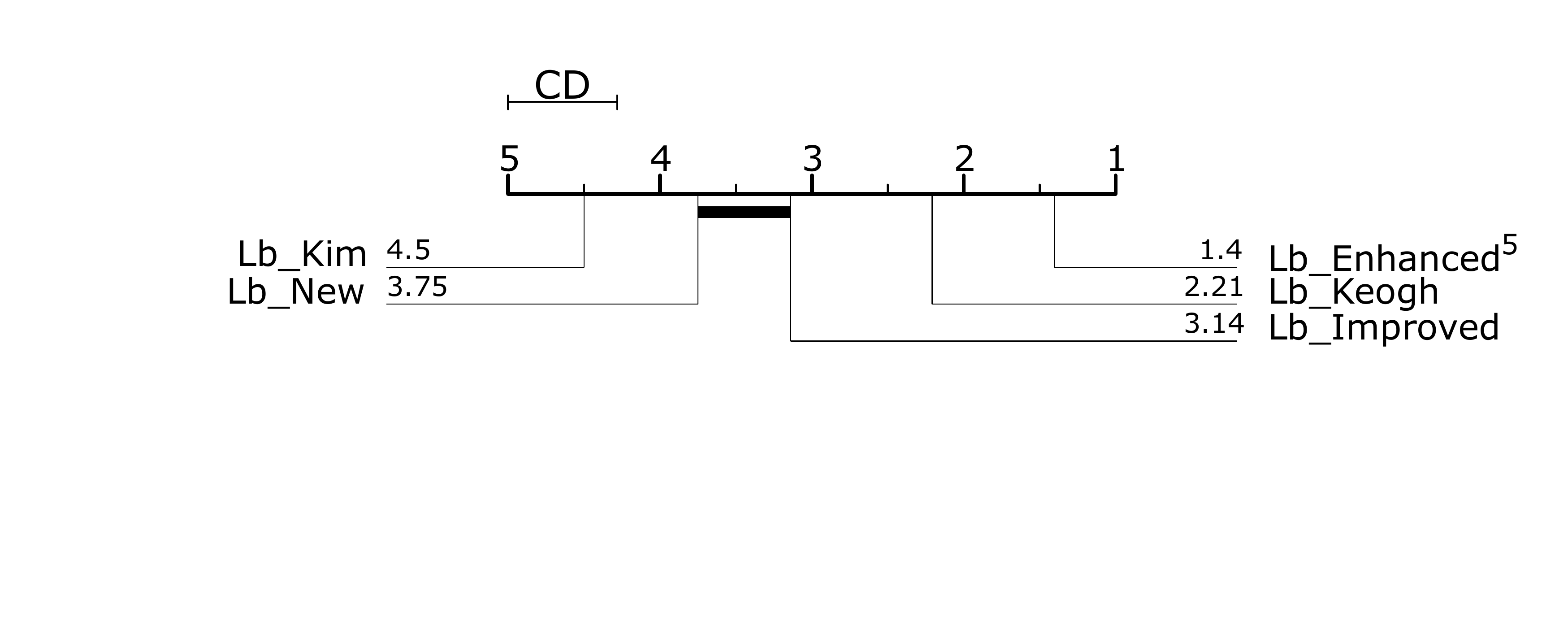}
        \caption{$W{=}6$}
        \label{fig:time w6}
    \end{subfigure}
    \begin{subfigure}[b]{0.9\columnwidth}
        \centering
        \includegraphics[width=\columnwidth, trim=70pt 165pt 20pt 0pt]{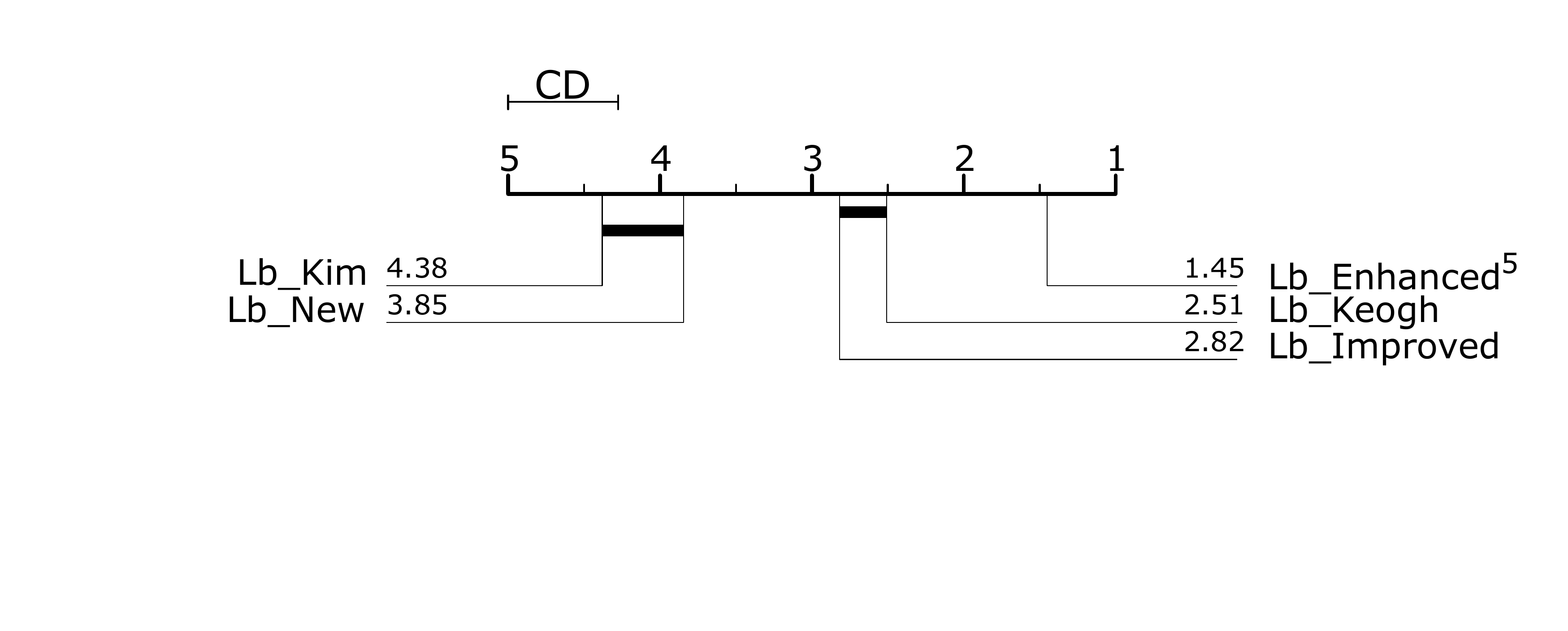}
        \caption{$W{=}10$}
        \label{fig:time w10}
    \end{subfigure}
    \begin{subfigure}[b]{0.9\columnwidth}
        \centering
        \includegraphics[width=\columnwidth, trim=70pt 160pt 20pt 0pt]{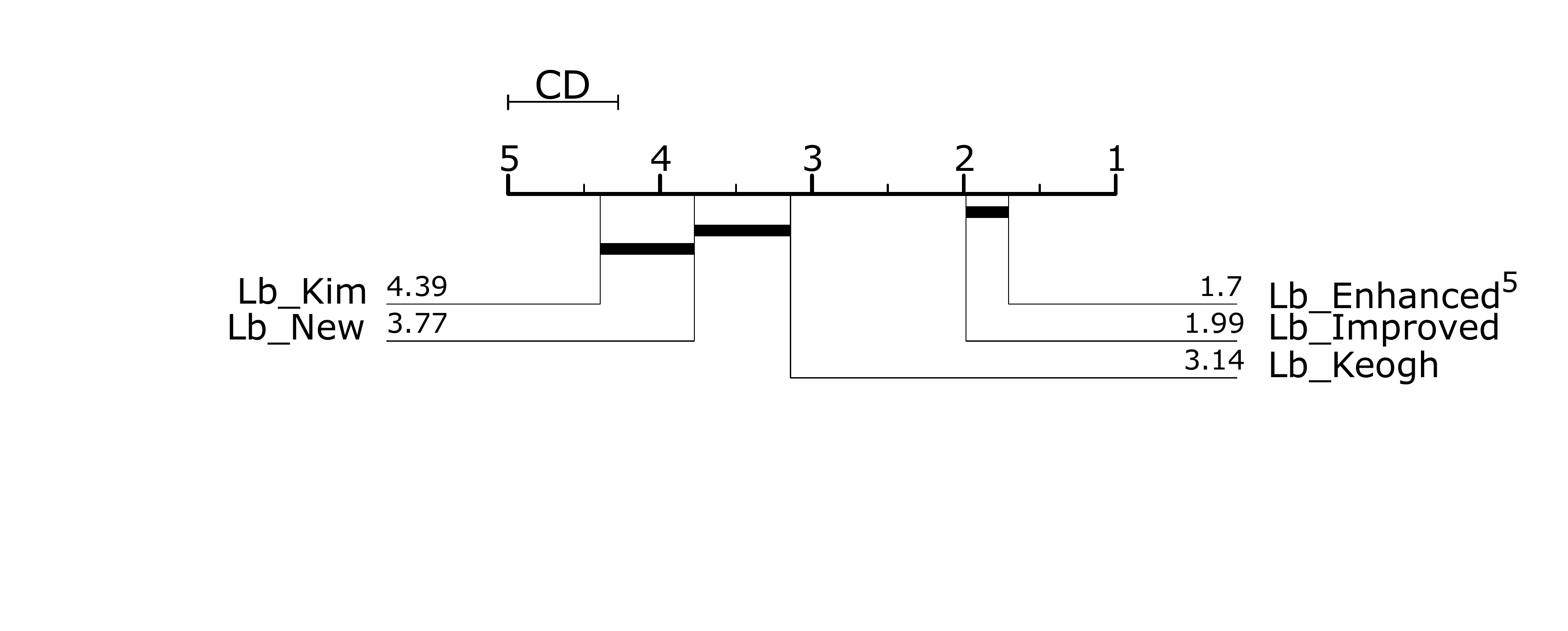}
        \caption{$W{=}0.1\cdot L$}
        \label{fig:time r10}
    \end{subfigure}
    \begin{subfigure}[b]{0.9\columnwidth}
        \centering
        \includegraphics[width=\columnwidth, trim=70pt 175pt 20pt 0pt]{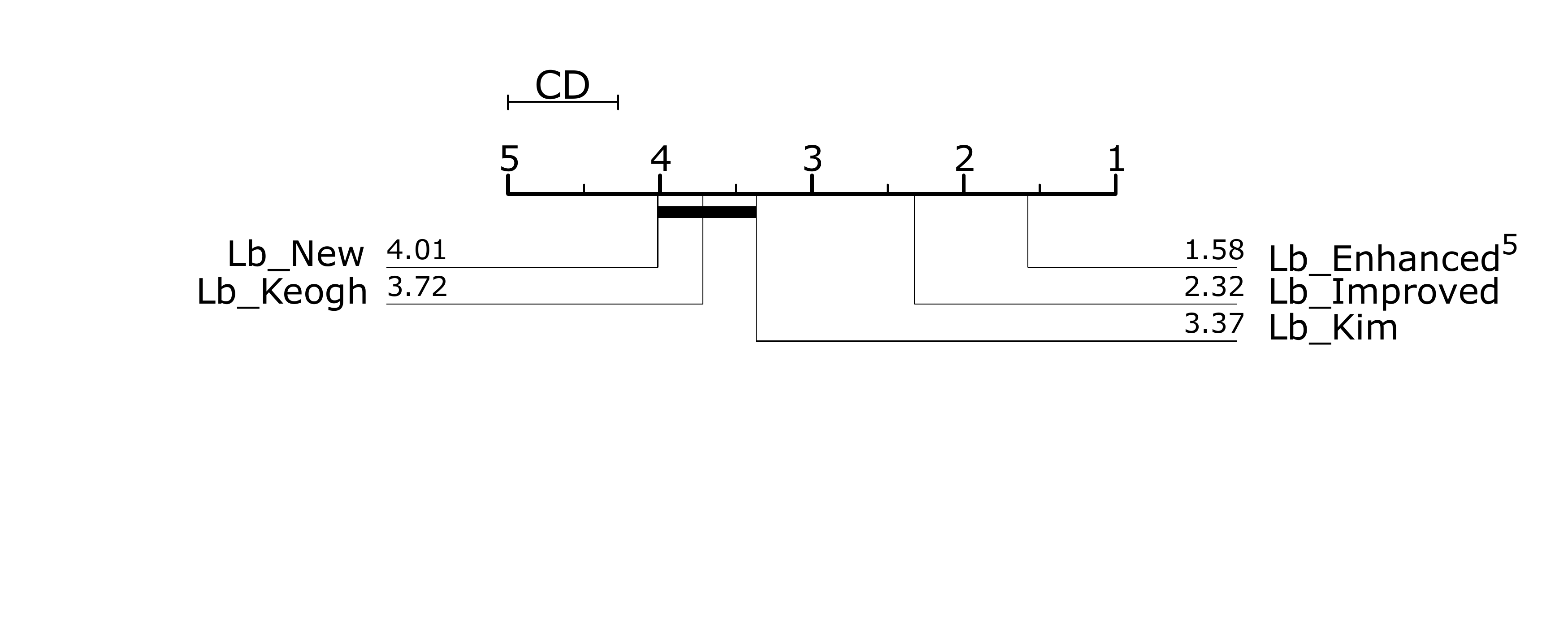}
        \caption{$W{=}0.5\cdot L$}
        \label{fig:time r50}
    \end{subfigure}
    \caption{Ranking of all lower bounds in terms of $\NNDTW$ classification time}
    \label{fig:cd time}
\end{figure}

Our \lbtanfive{} has the best average rank of all the lower bounds at all window sizes, significantly so at $W{=}6$ to 10 and $0.5\cdot L$ to $L$.
For smaller windows, \lbkeogh{} is the best of the remaining bounds. 
For larger windows, the tighter, but more computationally demanding \lbimproved{} comes to the fore.

We further extend our analysis by computing the speed up gained from  \lbtanfive{} relative to the other lower bounds. 
We compute the speedup of \lbtanfive{} to all the other lower bounds for $W=\{1,\ldots 10\}$ and present the geometric mean (average) in Figure~\ref{fig:geomean to enhanced5}.
The results show that \lbtanfive{} is consistently faster than all the other lower bounds. 

\begin{figure}[t]
\centering
\includegraphics[width=0.8\columnwidth,trim=10pt 0pt 10pt 10pt]{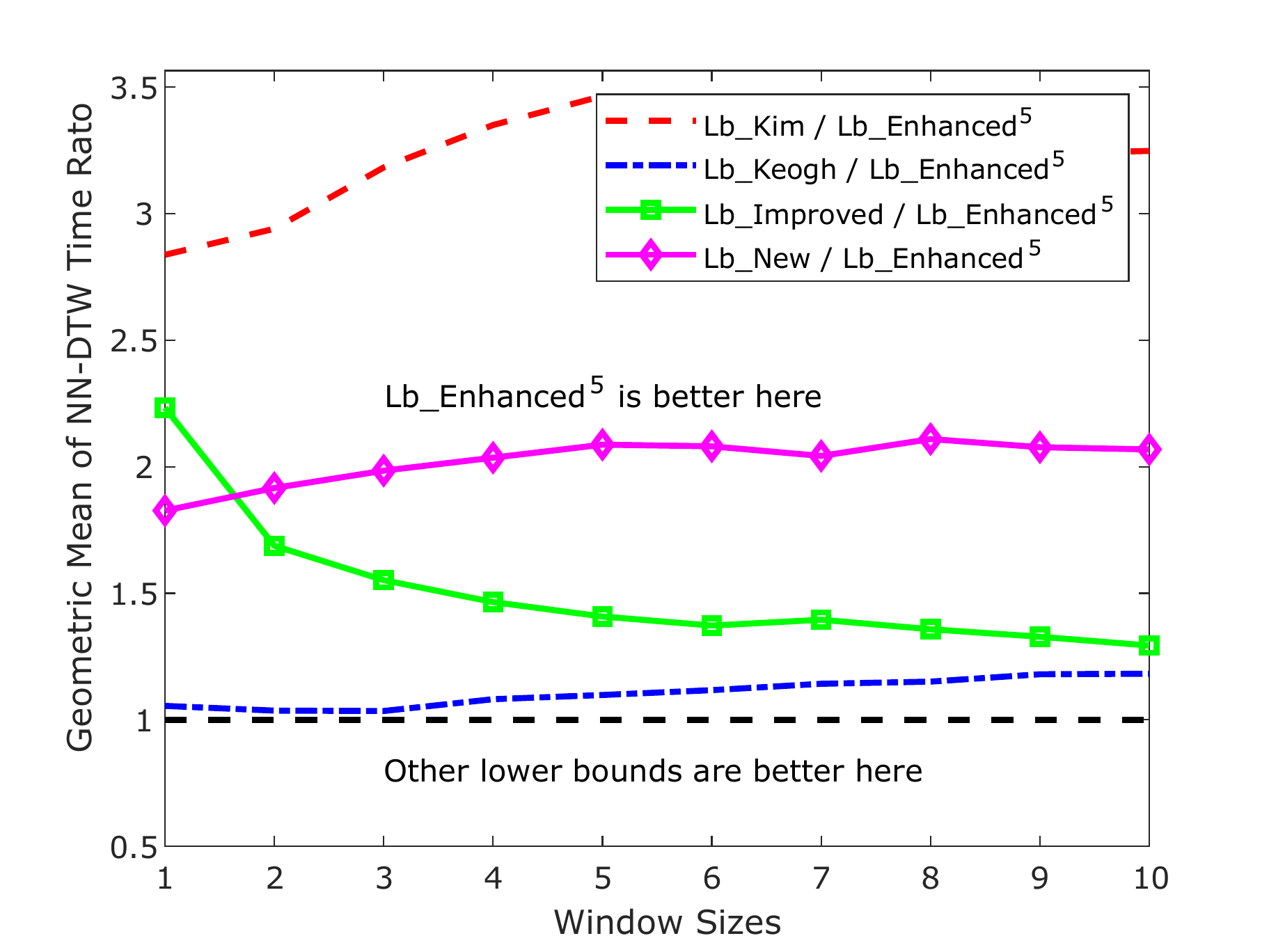}
\vspace*{-8pt}\caption{Geometric mean (average) ratio of classification time for major existing lower bounds to our proposed new lower bound \lbtanfive{} across all benchmark datasets \cite{ucrarchive} at $W=\{1,\ldots 10\}$.}
\label{fig:geomean to enhanced5}
\end{figure}



It might be thought that our experimental comparison has unfairly penalized \lbkeogh{} and \lbnew{} relative to \lbimproved{} and \lbtan{}, 
as only the latter use a form of early abandoning \cite{rakthanmanon2012searching}. 
However, \lbimproved{} starts with \lbkeogh{} and 
\lbtan{} uses \lbkeogh{} for most of the sequence.  
Hence, each of these could benefit as much as would \lbkeogh{} from the adoption of early abandoning in the \lbkeogh{} process.

\section{Conclusion and future work}
\label{sec:conclusion}
In conclusion, 
we proposed \lbtan{}, a new lower bound for $\DTW$.
The speed-tightness trade-off of \lbtan{} results in faster lower bound search for $\NNDTW$ than any of the previous established bounds at all window sizes. 
We expect it to be similarly effective at a wide range of nearest neighbor retrieval tasks under $\DTW$.
We showed that choosing a small tightness parameter $V$ is sufficient to effectively speed up $\NNDTW$.
Although it is possible to learn the best $V$ for a dataset (which will be future work),
our results show that when $V{=}5$, 
$\NNDTW{}$ with \lbtanfive{} is faster and more efficient than with the existing lower bounds for all warping window sizes across 85 benchmark datasets.

In addition, there is potential to replace \lbkeogh{} by \lbimproved{} within \lbtan{}.
This would increase the computation time, 
but the strong performance of \lbimproved{} suggests it should result in a powerful trade-off between time and tightness especially at larger windows.
Finally, since \lbkeogh{} is not symmetric with respect to $A$ and $B$, 
\lbtan{} is not symmetric too. 
Thus, $\max(\lbtan{}(A,B), \lbtan{}(B,A))$ is also a useful bound.
Our proposed lower bound could also be cascaded \cite{rakthanmanon2012searching}, 
which may further improve pruning efficiency.

\section{Acknowledgement}
This material is based upon work supported by the Air Force Office of Scientific Research, Asian Office of Aerospace Research and Development (AOARD) under award number FA2386-17-1-4036 and the Australian Research Council under award DE170100037. 
The authors would like to also thank Prof Eamonn Keogh and all the people who have contributed to the UCR time series classification archive.

\vspace{-7pt}
\bibliographystyle{siam}
{\tiny
\bibliography{biblio}
}

\end{document}